\pdfoutput=1

\documentclass[11pt]{article}

\usepackage{naacl2021}

\usepackage{times}
\usepackage{latexsym}

\usepackage[T1]{fontenc}

\usepackage[utf8]{inputenc}

\usepackage{microtype}

\usepackage{macros}
\usepackage{geometry}
\usepackage{graphicx}
\usepackage[nameinlink]{cleveref}

\usepackage{float}
\usepackage{booktabs}
\usepackage{makecell}
\usepackage{thmtools}
\usepackage{thm-restate}
\newcommand{\parsingmethod}[1]{\textsc{#1}}
\newcommand{\masktoken}{\texttt{[MASK]}\xspace}
\newcommand{\specific}{cloze-like\xspace}
\newcommand{\Specific}{Cloze-like\xspace}
\newcommand{\Specificp}[1]{\textsc{Cloze}-${#1}\%$\xspace}
\newcommand{\generic}{generic\xspace}

\newcommand{\masksequence}{X_{\setminus i}\xspace}
\newcommand{\masksequencetwo}{X_{\setminus \{i, j\}}\xspace}

\newcommand{\sigmazz}{\Sigma_{\rmZ\rmZ}\xspace}
\newcommand{\sigmaxx}{\Sigma_{\rmX\rmX}\xspace}
\newcommand{\sigmaxxnoi}{\Sigma_{\rmX\rmX, \setminus i, \setminus i}\xspace}
\newcommand{\sigmaxxtakei}{\Sigma_{\rmX\rmX, \setminus i, i}\xspace}

\newcommand{\xnoi}{\rmX_{\setminus i}\xspace}
\newcommand{\anoi}{A_{\setminus i}\xspace}
\newcommand{\atakei}{A_{i}\xspace}

\newcommand{\xmaskone}{\rvx_{\text{mask}, i}}
\newcommand{\xpca}{\rmX_{\text{PCA}}}
\newcommand{\betatwoslsi}{\beta_{\text{2SLS}, i}}
\newcommand{\betatwosls}{\beta_{\text{2SLS}}}
\newcommand{\betamaskone}{\beta_{\xnoi \rightarrow \rvx_i}}

\newcommand{\lmlm}{L_{\text{MLM}}}
\newcommand{\lfinetune}{L_{\text{finetune}}}

\definecolor{label_color}{HTML}{A95AA1}
\definecolor{specific_color}{HTML}{EB5757}
\definecolor{generic_color}{HTML}{2F80ED}

\def\rvx{{\mathbf{x}}}

\def\rmX{{\mathbf{X}}}
\def\rmY{{\mathbf{Y}}}
\def\rmZ{{\mathbf{Z}}}

\DeclareMathAlphabet{\mathsfit}{\encodingdefault}{\sfdefault}{m}{sl}
\SetMathAlphabet{\mathsfit}{bold}{\encodingdefault}{\sfdefault}{bx}{n}

\begin{document}
\title{On the Inductive Bias of Masked Language Modeling: \\ From Statistical to Syntactic Dependencies}
\author{
  Tianyi Zhang \\
  Computer Science Department \\
  Stanford University \\
  \texttt{tianyizhang@cs.stanford.edu} \\\And
  Tatsunori B. Hashimoto \\
  Computer Science Department \\
  Stanford University \\
  \texttt{thashim@stanford.edu} \\}

\maketitle

\begin{abstract}
We study how masking and predicting tokens in an unsupervised fashion can give  rise  to  linguistic  structures and downstream performance gains.  Recent theories have suggested that pretrained language models acquire useful inductive biases through masks that implicitly act as cloze reductions.
While appealing, we show that the success of the random masking strategy used in practice cannot be explained by such \specific masks alone. We construct cloze-like masks using task-specific lexicons for three different classification datasets and show that the majority of pretrained performance gains come from \emph{generic} masks that are not associated with the lexicon.
To  explain  the  empirical  success  of these generic masks,  we demonstrate a  correspondence  between  the  masked language model (MLM)  objective  and existing methods for learning statistical  dependencies in graphical models. Using this, we derive a method for extracting these learned statistical dependencies in MLMs and show that these dependencies encode useful inductive biases in the form of syntactic structures. In an unsupervised parsing evaluation, simply forming a minimum spanning tree on the implied statistical dependence structure outperforms a classic method for unsupervised parsing (58.74 vs. 55.91 UUAS).

\end{abstract}

\section{Introduction}
Pretrained masked language models~\citep{DeCh19ber,LiOt19rob} have benefitted a wide range of natural language processing (NLP) tasks~\citep{Li19fin, WaWe19ent,ZhXi20inc}.
Despite recent progress in understanding \emph{what} useful information is captured by MLMs~\citep{LiGa19lin,HeMa19str}, 
it remains a mystery \emph{why} task-agnostic masking of words can capture linguistic structures and transfer to downstream tasks.


\begin{figure}
    \centering
    \includegraphics[width=1.0\linewidth]{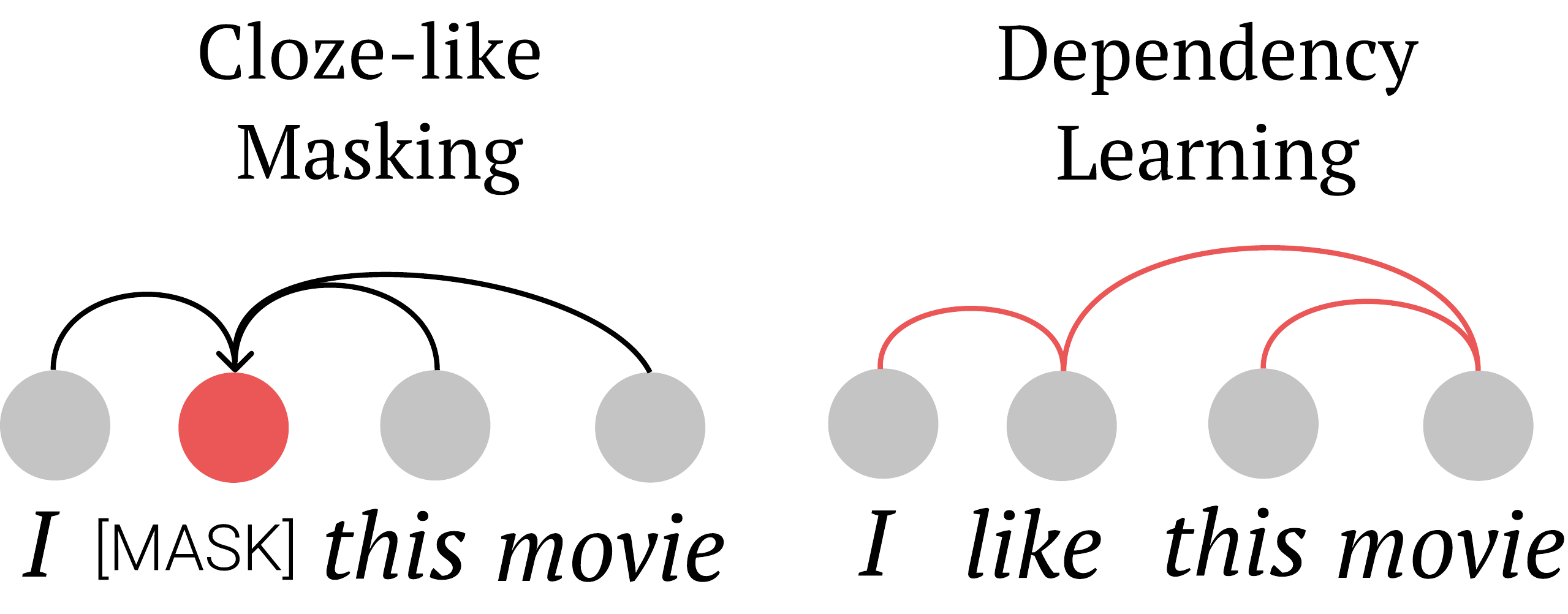}
    \caption{We study the inductive bias of MLM objectives and show that cloze-like masking (left) does not account for much of the downstream performance gains. Instead, we show that MLM objectives are biased towards extracting both statistical and syntactic dependencies using random masks (right).}
    \label{fig:figure_1}
    \vspace{-15pt}
\end{figure}

One popular justification of MLMs relies on viewing masking as a form of cloze reduction. Cloze reductions reformulate an NLP task into a prompt question and a blank and elicit answers by filling in the blank (\Cref{fig:figure_1}). 
When tested by cloze reductions pretrained MLMs and left-to-right language models (LMs) have been shown to possess abundant factual knowledge~\citep{PeRo19lan} and display impressive few-shot ability~\citep{BrMa20lan}.
This success has inspired recent hypotheses that some word masks are \specific and provide indirect supervision to downstream tasks~\citep{SaMa20mat,LeLe20pre}.
For example, a sentiment classification task~\citep{PaLe02thu} can be reformulated into filling in \textit{like} or \textit{hate} in the cloze \textit{I \masktoken this movie}. Such \specific masks provide a clear way in which an MLM can implicitly learn to perform sentiment classification.

While this hypothesis is appealing, MLMs in practice are trained with uniform masking that does not contain the special structure required by \specific masks most of the time.
For example, predicting a \generic word \textit{this} in the cloze \textit{I like \masktoken movie} would not offer task-specific supervision.
We quantify the importance of \specific and \generic masks by explicitly creating \specific masks using task-specific lexicons and comparing models pretrained on these masks. These experiments suggest that although \specific masks can be helpful, the success of uniform masking cannot be explained via \specific masks alone. In fact, we demonstrate that uniform masking performs as well as a negative control where we explicitly remove \specific masks from the mask distribution.

To address this mismatch between theory and practice, we offer a new hypothesis of how \generic masks can help downstream learning.
We propose a conceptual model for MLMs by drawing a correspondence between masking and graphical model neighborhood selection~\citep{MeBh06hig}.
Using this, we show that MLM objectives are designed to recover statistical dependencies in the presence of latent variables and propose an estimator that can recover these learned dependencies from MLMs. 
We hypothesize that statistical dependencies in the MLM objective capture useful linguistic dependencies and demonstrate this by using recovered statistical dependencies to perform unsupervised parsing, outperforming an actual unsupervised parsing baseline~\citep[$58.74$ vs $55.91$ UUAS;][]{KlMa04cor}. 
We release our implementation on Github\footnote{\url{https://github.com/tatsu-lab/mlm_inductive_bias}}.

\section{Related works}

\textbf{Theories inspired by Cloze Reductions.}
Cloze reductions are fill-in-the-blank tests that reformulate an NLP task into an LM problem. Existing work demonstrates that such reductions can be highly effective for zero/few-shot prediction~\citep{RaWu19lan,BrMa20lan} as well as relation extraction~\citep{PeRo19lan,JiXu20how}.

These fill-in-the-blank tasks provide a clear way by which LMs can obtain supervision about downstream tasks, and recent work demonstrates how such implicit supervision can lead to useful representations~\citep{SaMa20mat}.
More general arguments by \citet{LeLe20pre} show these theories hold across a range of self-supervised settings.
While these theories provide compelling arguments for the value of pre-training with cloze tasks, they do not provide a clear reason why \emph{uniformly random} masks such as those used in BERT provide such strong gains. 
In our work, we quantify this gap using lexicon-based \specific masks and show that \specific masks alone are unlikely to account for the complete success of MLM since \generic and non-cloze masks are responsible for a substantial part of the empirical performance of MLMs.

\textbf{Theories for vector representations.}
Our goal of understanding how masking can lead to useful inductive biases and linguistic structures is closely related to that of papers studying the theory of word embedding representations~\citep{MiCh13eff,PeSo14glo, ArLi15ran}.
Existing work has drawn a correspondence between word embeddings and low-rank factorization of a pointwise mutual information (PMI) matrix~\citep{LeGo14neu} and others have shown that PMI is highly correlated with human semantic similarity judgements~\citep{HaAl16wor}.

While existing theories for word embeddings cannot be applied to MLMs, we draw inspiration from them and derive an analogous set of results. Our work shows a correspondence between MLM objectives and graphical model learning through conditional mutual information, as well as evidence that the conditional independence structure learned by MLMs is closely related to syntactic structure.



\textbf{Probing Pretrained Representations.}
Recent work has applied probing methods~\citep{BeGl19ana} to analyze \emph{what} information is captured in the pretrained representations.
This line of work shows that pretrained representations encode a diverse range of knowledge~\citep{PeNe18dis,TeXi19wha,LiGa19lin,HeMa19str,WuCh20per}.
While probing provides intriguing evidence of linguistic structures encoded by MLMs, they do not address the goals of this work, which is \emph{how} the pretraining objective encourages MLMs to extract such structures.
\section{Motivation}
\subsection{Problem Statement}
\textbf{Masked Language Modeling} asks the model to predict a token given its surrounding context. 
Formally, consider an input sequence $X$ of $L$ tokens $\langle x_1, \ldots, x_L  \rangle$ where each variable takes a value from a vocabulary $\mathcal{V}$.
Let $X\sim D$ be the data generating distribution of $X$.
Let $x_i$ be the $i$th token in $X$, and let $\masksequence$ denote the sequence after replacing the $i$th token with a special \masktoken token. 
In other words, 
$$\masksequence := \langle x_1, \ldots, x_{i-1}, \text{\masktoken}, x_{i+1}, \ldots, x_L  \rangle.$$
Similarly, define $\masksequencetwo$ as replacing both $x_i$ and $x_j$ with \masktoken.
MLM determines what tokens are masked by a mask distribution $i \sim M$.
The goal of MLM is to learn a probabilistic model $p_{\theta}$ that minimizes
$$
L_{\text{MLM}} = \underset{X \sim D, i \sim M}{\E} \;-\log p_{\theta}(x_i | \masksequence).
$$
In BERT pretraining, each input token is masked with a fixed, uniform probability, which is a hyperparameter to be chosen.
We refer to this strategy as \emph{uniform masking}.

\textbf{Finetuning} is the canonical method for using pretrained MLMs.
Consider a prediction task where $y \in \mathcal{Y}$ is the target variable, \emph{e.g.}, the sentiment label of a review.
Finetuning uses gradient descent to modify the pretrained parameters $\theta$ and learn a new set of parameters $\phi$ to minimize
$$
L_{\text{finetune}} = \underset{X\sim D', y\sim p(y|X)}{\E} \; - \log p_{\theta, \phi}(y | X),
$$
where $p(y|x)$ is the ground-truth distribution and $D'$ is the data distribution of the downstream task.

\textbf{Our goals.}
We will study how the mask distribution $M$ affects downstream performance.
We define perfect cloze reductions as some partition of the vocabulary  $\mathcal{V}_y$ such that $p(x_{i} \in \mathcal{V}_y | \masksequence) \approx p(y|X)$.
For a distribution $M$ such that the masks we draw are perfect cloze-reductions, the MLM objective offers direct supervision to finetuning since \ $L_{\text{MLM}} \approx L_{\text{finetune}}$.
In contrast to \specific masking, in uniform masking we can think of $p_{\theta}$ as implicitly learning a generative model of $\rmX$~\citep{WaCh19ber}.
Therefore, as $M$ moves away from the ideal distribution and becomes more uniform, we expect $p_{\theta}$ to model more of the full data distribution $D$ instead of focusing on \specific supervision for the downstream task. This mismatch between theory and practice raises questions about how MLM with uniform masking can learn useful inductive biases.

When $\lmlm$ is not $\lfinetune$, what is $\lmlm$ learning?
We analyze $\lmlm$ and show that it is similar to a form of conditional mutual information based graphical model structure learning.

\subsection{Case Study for \Specific Masking}
\begin{figure}[t!]
    \centering
    \includegraphics[width=0.85\linewidth]{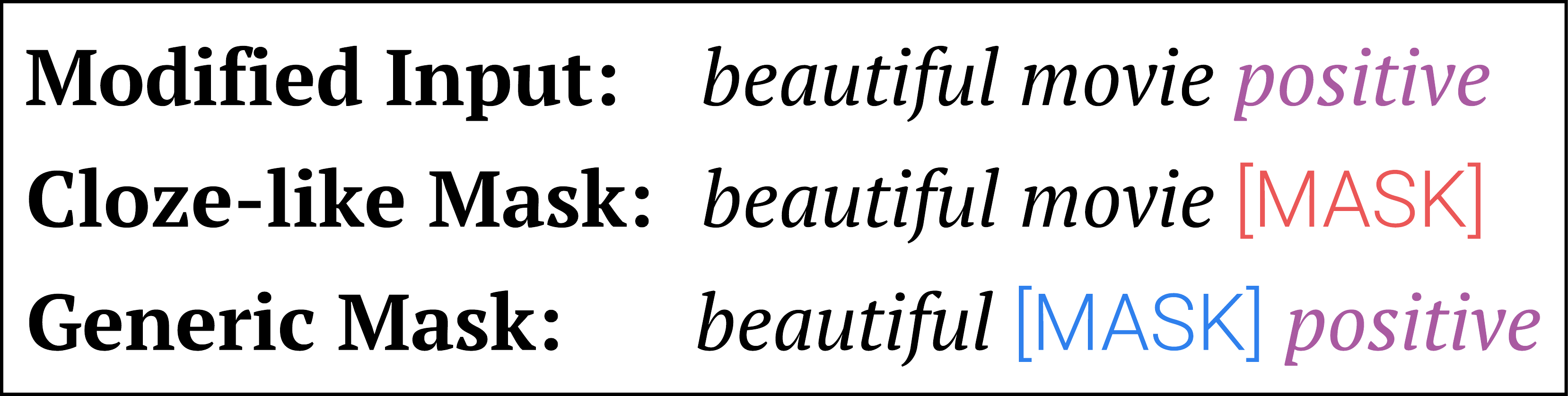}
    \caption{In our case study, we append the \textcolor{label_color}{true label} to each input and create ideal \textcolor{specific_color}{\specific masks}. We study how deviations from the ideal mask distribution affect downstream performance by adding in \textcolor{generic_color}{\generic masks}.
    }
    \label{fig:case_study_illustration}
    \vspace{-10pt}
\end{figure}

\label{sec:case_study}

To motivate our subsequent discussions, we perform a controlled study for the case when $\lmlm \approx \lfinetune$ and analyze how deviations from the ideal mask distribution affect downstream performance. 
We perform analysis on the Stanford Sentiment Treebank~\citep[SST-2;][]{SoPe13rec}, which requires models to classify short movie reviews into positive or negative sentiment.
We append the ground-truth label (as the word \textit{positive} or \textit{negative}) to each movie review (\Cref{fig:case_study_illustration}).
Masking the last word in each review is, by definition, an ideal mask distribution.
To study how the deviation from the ideal mask distribution degrades downstream performance, we vary the amount of \specific masks during training.
We do this by masking out the last word for $p\%$ of the time and masking out a random word in the movie review for $(100-p)\%$ of the time, and choose $p \in \{0, 20, 40, 60, 80, 100\}$.

\textbf{Experimental details.} 
We split the SST-2 training set into two halves, use one for pretraining, and the other for finetuning.
For the finetuning data, we do not append the ground-truth label.
We pretrain small transformers with $\lmlm$ using different masking strategies and finetune them along with a baseline that is not pretrained (\textsc{NoPretrain}). 
Further details are in \Cref{app:details}.

\begin{figure}[t!]
    \centering
    \includegraphics[width=\linewidth]{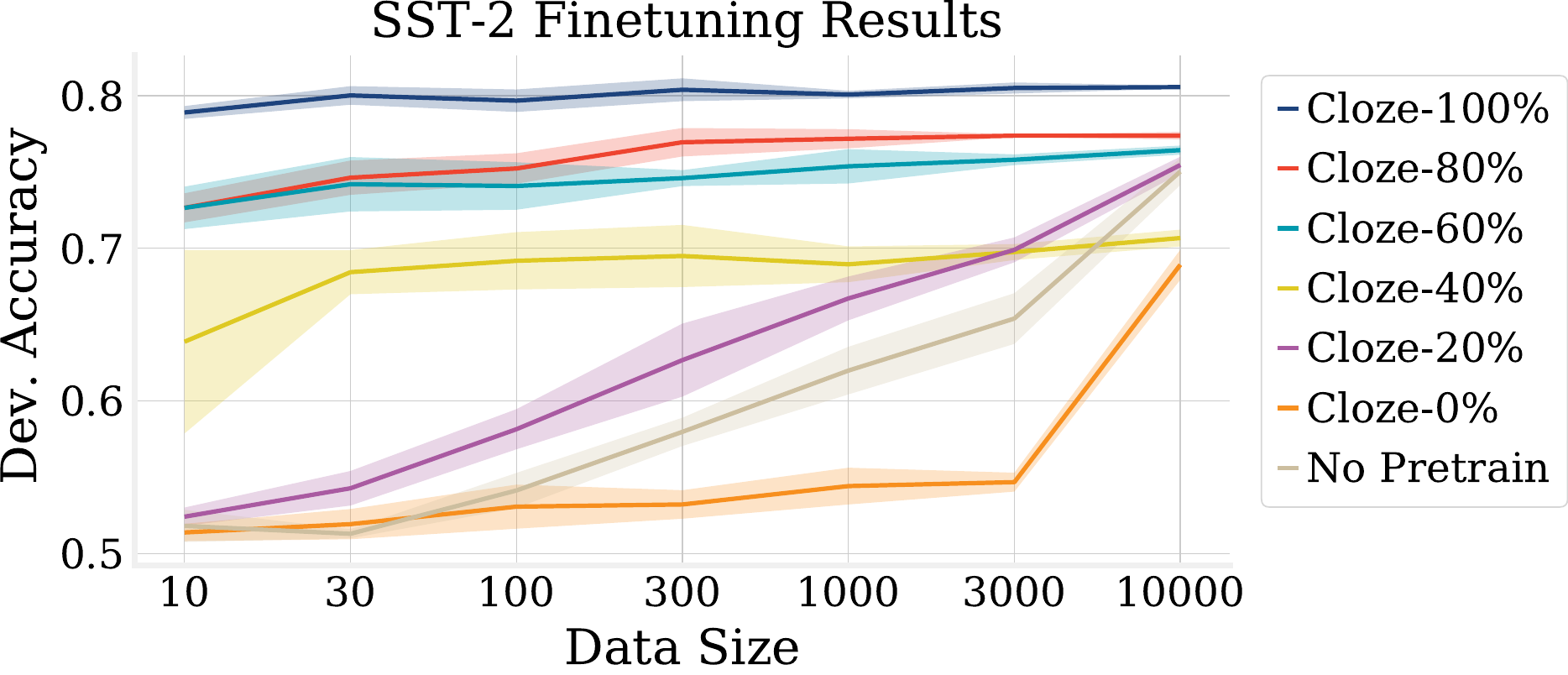}
    \caption{SST-2 development set accuracy. \Specificp{p} is pretrained on a mixture of masks where $p\%$ of the masks are \Specific. \textsc{NoPretrain} trains a classifier without any pretraining.
    Even a small modification of the ideal mask distribution degrades performance.
    }
    \label{fig:sst_pretrain_finetune}
    \vspace{-15pt}
\end{figure}

\textbf{Results.}
We observe that while \specific masks can lead to successful transfer, \emph{even a small modification of the ideal mask distribution deteriorates performance}. 
\Cref{fig:sst_pretrain_finetune} shows the development set accuracy of seven model variants averaged across ten random trials.
We observe as $p$ decreases, the performance of \Specificp{p} degrades.
Notably, \Specificp{80} is already worse than \Specificp{100} and \Specificp{20} does not outperform \textsc{NoPretrain} by much.
We notice that \Specificp{0} in fact degrades finetuning performance, potentially because the pretrained model is over-specialized to the language modeling task~\citep{ZhWu20rev,TaSi20inv}.
While this is a toy example, we observe similar results for actual MLM models across three tasks (\Cref{sec:random_mask}), and this motivates us to look for a framework that explains the success of \generic masks in practice.

\section{Analysis}
In the previous section, we saw that \specific masks do not necessarily explain the empirical success of MLMs with uniform masking strategies. Understanding uniform masking seems challenging at first, as uniform-mask MLMs seem to lack task-specific supervision and is distinct from existing unsupervised learning methods such as word embeddings (which rely upon linear dimensionality reduction) and autoencoders (which rely upon denoising). However, we show in this section that there is a correspondence between MLM objectives and classic methods for graphical model structure learning. As a consequence, we demonstrate that MLMs are implicitly trained to recover statistical dependencies among observed tokens.

\subsection{Intuition and Theoretical Analysis}
\label{sec:gaussian_theory}
Our starting point is the observation that predicting a single feature ($x_i$) from all others ($\masksequence$) is the core subroutine in the classic Gaussian graphical model structure learning algorithm of \citet{MeBh06hig}. In this approach, $L$ different Lasso regression models are trained~\citep{Ti96reg} with each model predicting $x_i$ from $\masksequence$, and the nonzero coefficients of this regression correspond to the conditional dependence structure of the graphical model. 

\begin{figure}[t]
    \centering
    \includegraphics[width=0.6\linewidth]{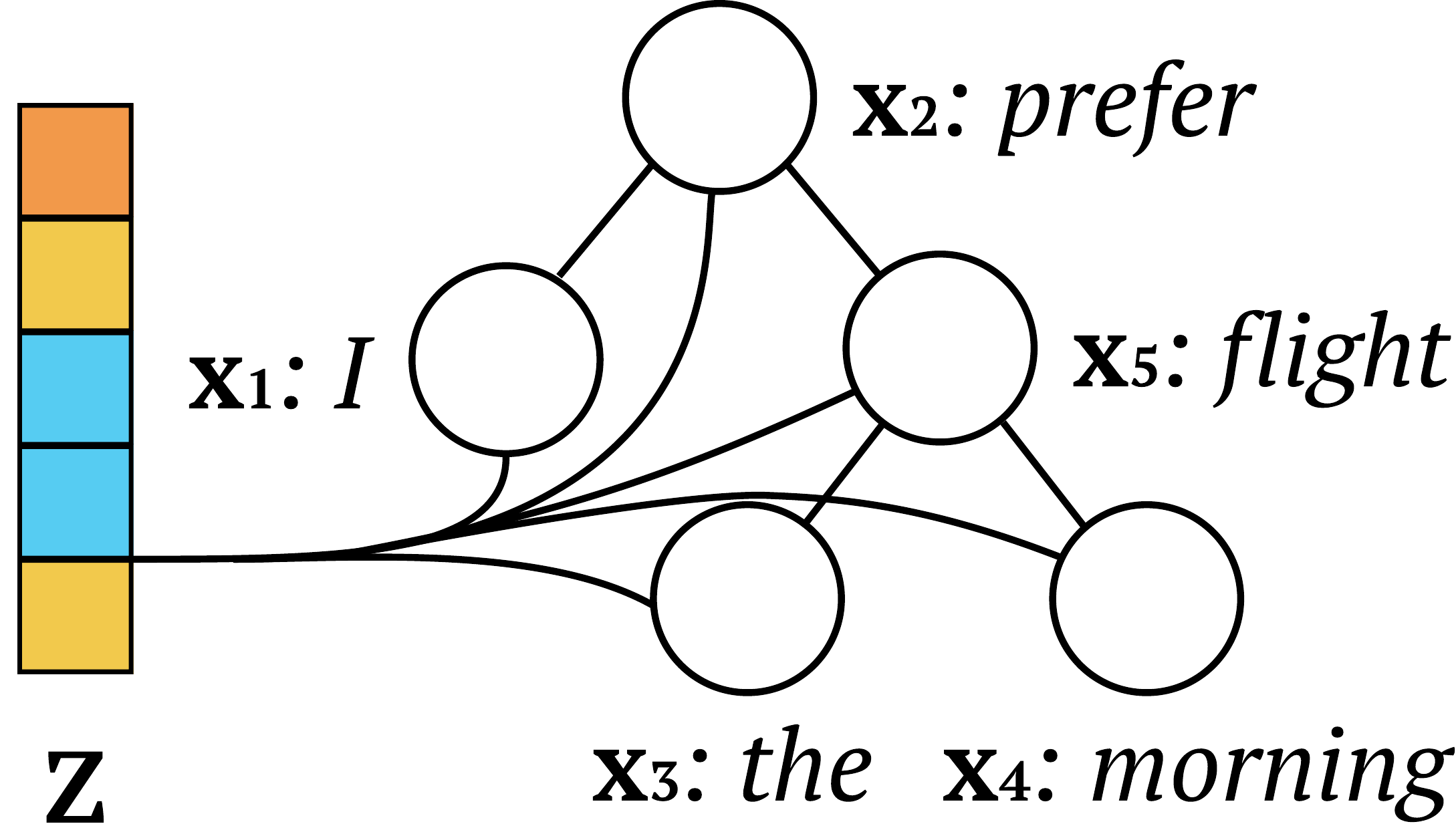}
    \caption{Our conceptual framework of MLM. All coordinates of $X$ are dependent on the latent variable $\rmZ$ while there is only sparse dependency among $X$.}
    \label{fig:theory_illustration}
    \vspace{-15pt}
\end{figure}

The MLM objective can be interpreted as a nonlinear extension of this approach, much like a classical algorithm that uses conditional mutual information (MI) estimators to recover a graphical model~\citep{AnTa12hig}. 
Despite the similarity, real world texts are better viewed as models with latent variables \citep[\emph{e.g.} topics;][]{BlNg03lat} and many dependencies across tokens arise due to latent variables, which makes learning the direct dependencies difficult.
We show that MLMs implicitly recover the latent variables and can capture the direct dependencies while accounting for the effect of latent variables.
Finally, MLMs are only approximations to the true distribution and we show that the MLM objective can induce high-quality approximations of conditional MI.



\textbf{Analysis setup.} To better understand MLMs as a way to recover graphical model structures, we show mask-based models can recover latent variables and the direct dependencies among variables in the Gaussian graphical model setting of~\citet{MeBh06hig}.
Let $\rmX = [\rvx_1, \ldots, \rvx_L] \in \mathbb{R}^{L}$ represent an input sequence where each of its coordinates $\rvx_i$ represents a token, and $\rmZ \in \mathbb{R}^{k}$ be a latent variable that controls the sequence generation process. 
We assume that all coordinates of $\rmX$ are dependent on the latent variable $\rmZ$, and there are sparse dependencies among the observed variables (\Cref{fig:theory_illustration}).
In other words, we can write $\rmZ \sim \normal(0, \Sigma_{\rmZ\rmZ})$ and $\rmX \sim \normal(A \rmZ, \sigmaxx)$.
Intuitively, we can imagine that $\rmZ$ represents shared semantic information, \emph{e.g.}\ a topic, and $\sigmaxx$ represents the syntactic dependencies. In this Gaussian graphical model, the MLM is analogous to regressing each coordinate of $\rmX$ from all other coordinates, which we refer to as masked regression.

\textbf{MLM representations can recover latent variable.} 
We now study the behavior of masked regression through the representation $\xmaskone$ that is obtained by applying masked regression on the $i$th coordinate of $\rmX$ and using the predicted values.
Our result shows that masked regression is similar to the two-step process of first recovering the latent variable $\rmZ$ from $\xnoi$ and then predicting $x_i$ from $\rmZ$.

Let $\sigmaxxtakei \in \mathbb{R}^{d{-}1}$ be the vector formed by dropping the $i$th row and taking the $i$th column of $\sigmaxx$ and $\betatwoslsi$ be the linear map resulting from the two-stage regression $\xnoi \to \rmZ \to x_i$.

\begin{restatable}{prop}{regressionprop}
\label{prop:regression}
Assuming that $\sigmaxx$ is full rank, 
$$
    \xmaskone = \betatwoslsi \xnoi + O(\norm{\sigmaxxtakei}_{2}),
$$
\end{restatable}
In other words, masked regression implicitly recovers the subspace that we would get if we first explicitly recovered the latent variables ($\betatwoslsi$) with an error term that scales with the off-diagonal terms in $\Sigma_{XX}$. The proof is presented in \Cref{app:proof}.

To give additional context for this result, let us consider the behavior of a different representation learning algorithm: PCA. It is well-known that PCA can recover the latent variables as long as the $\sigmazz$ dominates the covariance $Cov(\rmX)$. 
We state this result in terms of $\xpca$, the observed data projected to the first $k$ components of PCA.
\begin{restatable}{prop}{pcaprop}
\label{prop:pca}
    Let $\lambda_k$ be the $k$th eigenvalue of $A\sigmazz A^{\top}$ and $\lambda_{\rmX\rmX, k{+}1}$ be the $k{+}1$th eigenvalue of $\sigmaxx$ and $V$ be the first $k$ eigenvectors of Cov(X). Assuming $\lambda_k > \lambda_{\rmX\rmX, k{+}1}$, we have
    {
        \small
    \begin{multline*}
    \E_{\rmX} \norm{A\rmZ - \xpca}_{2} \leq \\ 
        \frac{\sqrt{2}\norm{\sigmaxx}_{op}}{\lambda_{k} - \lambda_{\rmX\rmX, k+1}} (\norm{A\rmZ}_{2} + \sqrt{\mathrm{tr}(\sigmaxx)}) +
        \norm{AA^{\top}}_{op} \sqrt{\mathrm{tr}(\sigmaxx)},
    \end{multline*}
    }
    where $\norm{\cdot}_{op}$ is the operator norm and $\mathrm{tr}(\cdot)$ is the trace.
  \end{restatable}
This shows that whenever $\Sigma_{XX}$ is sufficiently small and $\lambda_k$ is large (i.e., the covariance is dominated by $Z$), then PCA recovers the latent information in $\rmZ$.
The proof is based on the Davis-Kahan theorem~\citep{StSu90mat} and is presented in \Cref{app:proof}.

Comparing the bound of PCA and masked regression, both bounds have errors that scales with $\sigmaxx$, but the key difference in the error bound is that the error term for masked regression does not scale with the per-coordinate noise ($\diag(\sigmaxx)$) and thus can be thought of as focusing exclusively on interactions within $\rmX$.
Analyzing this more carefully, we find that $\sigmaxxtakei$ corresponds to the statistical dependencies between $x_i$ and $\xnoi$, which we might hope captures useful, task-agnostic structures such as syntactic dependencies. 


\textbf{MLM log-probabilies can recover direct dependencies.}
Another effect of latent variables is that many tokens have indirect dependencies through the latent variables, which poses a challenge to recovering the direct dependencies among tokens.
We now show that the MLMs can account for the effect of latent variable.

In the case where there are no latent variables, we can identify the direct dependencies via conditional MI~\citep{AnTa12hig} because any $x_i$ and $x_j$ that are disconnected in the graphical model will have zero conditional MI, \emph{i.e.}, $I(x_i; x_j | \masksequencetwo)=0$.
One valuable aspect of MLM is that we can identify direct dependencies even in the presence of latent variables.

If we naively measure statistical dependency by mutual information, the coordinates of $\rmX$ would appear dependent on each other because they are all connected with $\rmZ$.
However, the MLM objective resolves this issue by conditioning on $\masksequencetwo$.
We show that latent variables (such as topics) that are easy to predict from $\masksequencetwo$ can be ignored when considering conditional MI.
\begin{restatable}{prop}{entropyprop}
  \label{prop:mi_entropy}
  The gap between conditional MI with and without latent variables is bounded by the conditional entropy $H(\rmZ|\masksequencetwo)$,
    \begin{multline*}
        I(x_i; x_j | \masksequencetwo) - I(x_i; x_j | \rmZ, \masksequencetwo) \\
        \leq 2H(\rmZ | \masksequencetwo).
    \end{multline*}
\end{restatable}
This suggests that when the context $\masksequencetwo$ captures enough of the latent information, conditional MI can remove the confounding effect of the shared topic $\rmZ$ and extract the direct and sparse dependencies within $\rmX$ (see \Cref{app:proof} for the proof).



\textbf{MLM objective encourages capturing conditonal MI.}
We have now shown that conditional MI captures direct dependencies among tokens, even in the presence of latent variables. 
Next, we will show that the MLM objective ensures that a LM with low log-loss accurately captures the conditional MI.
We now show that learning the MLM objective implies high-quality estimation of conditional MI. 
Denote $X(i, v)$ as substituting $x_i$ with a new token $v$,
\begin{multline}
    X(i, v) = \langle x_1, \ldots, x_{i-1}, v, x_{i+1}, 
     \ldots, x_L  \rangle. \nonumber
\end{multline}
Conditional MI is defined as the expected pointwise mutual information (PMI) conditioned on the rest of the tokens,
{
\small
$$
   I_{p}
   =\; \underset{x_i, x_j}{\E} [\; \log p(x_i | \masksequence(j, x_j))
- \log \underset{x_j | x_i}{\E} p(x_i | \masksequence(j, x_j))\;]
$$
}
where $I_{p}$ is the abbreviation of $I_{p}(x_i;x_j|\masksequencetwo)$.
Our main result is that the log-loss MLM objective directly bounds the gap between the true conditional mutual information from the data distribution and an estimator that uses the log-probabilities from the model. More formally,
\begin{restatable}{prop}{estimatorprop}
    \label{prop:estimator}
    Let 
    {
        \small
        $$
        \hat{I}_{p_{\theta}} =\underset{x_i, x_j}{\E} [\log p_{\theta}(x_i | \masksequence(j, x_j))
        - \log \underset{x_j | x_i}{\E} p_{\theta}(x_i | \masksequence(j, x_j))]
        $$
    }
    be an estimator constructed by the model distribution $p_{\theta}$.
    Then we can show,
    {
        \small
        $$
        \lvert \hat{I}_{p_{\theta}}-I_{p} \rvert
        \leq \underset{x_j}{\E}\; \dkl{p(x_i | \masksequence(j, x_j))}{p_{\theta}(x_i | \masksequence(j, x_j))},
        $$
    }
    where $D_{\rm kl}$ represents the KL-divergence.
\end{restatable}

Here, the KL-divergence corresponds to the $\lmlm$ objective, up to a constant entropy term that depends on $p$. We present the proof in \Cref{app:proof}. In other words, the MLM objective is implicitly encouraging the model to match its implied conditional MI to that of the data. We now use this result to create an estimator that extracts the conditional independence structures implied by MLM.

\subsection{Extracting statistical dependencies implied by MLMs}
Our earlier analysis in~\Cref{prop:estimator} suggests that an MLM with low loss has an accurate approximation of conditional mutual information. Using this result, we will now propose a procedure which estimates $\hat{I}_{p_{\theta}}$.
The definition of $\hat{I}_{p_{\theta}}$ shows that if we can access samples of $x_i$ and $x_j$ from the true distribution $p$, then we can directly estimate the conditional mutual information by using the log probabilities from the MLM. Unfortunately, we cannot draw new samples of $x_j \mid \masksequencetwo$, leading us to approximate this distribution using Gibbs sampling on the MLM distribution.

Our Gibbs sampling procedure is similar to the one proposed in \citet{WaCh19ber}.
We start with $X^{0} =\; \masksequencetwo$.
For the $t$th iteration, we draw a sample $x^{t}_{i}$ from $p_{\theta}(x_i | \masksequence^{t-1})$, and update by $X^{t} = X^{t-1}(i, x^{t}_{i})$.
Then, we draw a sample $x^{t}_{j}$ from $p_{\theta}(x_j | X^{t}_{\setminus j})$ and set $X^{t} = X^{t}(j, x^{t}_{j})$.
We repeat and use the samples $(x^{1}_{i}, x^{1}_{j}), \ldots, (x^{t}_{i}, x^{t}_{j})$ to compute the expectations for conditional MI.

This procedure relies upon an additional assumption that samples drawn from the MLM are faithful approximations of the data generating distribution. 
However, we show empirically that even this approximation is sufficient to test the hypothesis that the conditional independences learned by an MLM capture syntactic dependencies (\Cref{sec:parsing}).
\section{Experiment}
We now test two predictions from our analyses.
First, similar to our observation in the case study, we show that \specific masks do not explain the success of uniform masks on three real-world datasets.
Second, our alternative view of relating MLM to graphical models suggests that statistical dependencies learned by MLMs may capture linguistic structures useful for downstream tasks. We demonstrate this by showing that MLMs' statistical dependencies reflect syntactic dependencies.

\subsection{Uniform vs Cloze-like Masking}
\label{sec:random_mask}

\begin{figure*}[ht!]
    \centering
    \includegraphics[width=\textwidth]{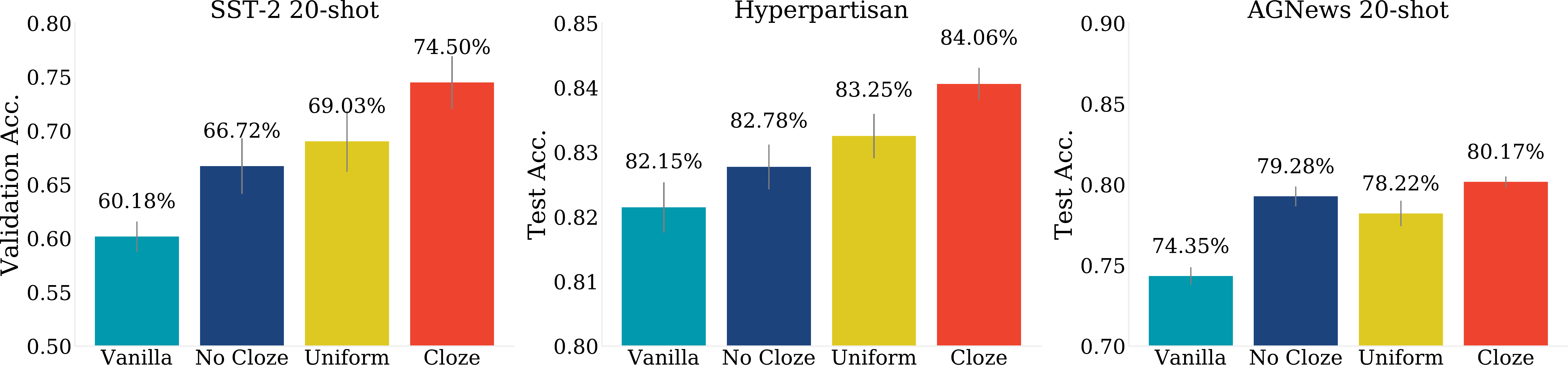}
    \caption{Finetuning performance with different masking strategies averaged across twenty random trials and error bars showing $95\%$ confidence intervals.
    \textsc{Vanilla} represents a BERT model without any second-stage pretraining.
    \textsc{Cloze} and \textsc{NoCloze} represent models train with or without \specific masks, respectively.
    \textsc{Uniform} uses the uniform random masking strategy proposed in \citet{DeCh19ber} for second-stage pretraining.
    }
    \label{fig:finetune_results}
    \vspace{-10pt}
\end{figure*}

\textbf{Setup.} 
We now demonstrate that real-world tasks and MLMs show a gap between task-specific cloze masks and random masks.
We compare the MLM with random masking to two different control groups. 
In the positive control (\textsc{Cloze}), we pretrain with only \specific masks and in the negative control (\textsc{NoCloze}), we pretrain by explicitly excluding \specific masks.
If the success of MLM can be mostly explained by implicit cloze reductions, then we should expect \textsc{Cloze} to have strong downstream performance while \textsc{NoCloze} leads to a minimal performance gain. 
We compare pretraining with the uniform masking strategy used in BERT (\textsc{Uniform}) to these two control groups.
If \textsc{Uniform} performs worse than the positive control and more similar to the negative control, then we know that uniform masking does not leverage \specific masks effectively.

\textbf{Simulating Pretraining.} 
Given computational constraints, we cannot retrain BERT from scratch.
Instead, we approximate the pretraining process by continuing to update BERT with MLM~\citep{GuMa20don}, which we refer to as second-stage pretraining.
Although this is an approximation to the actual pretraining process, the second-stage pretraining shares the same fundamental problem for pretraining: how can unsupervised training lead to downstream performance gains?

We study the effectiveness of different masking strategies by comparing to a BERT model without second-stage pretraining (\textsc{Vanilla}).
We experiment with three text classification datasets: SST-2~\citep{SoPe13rec}, Hyperpartisan~\citep{KiMe19sem}, and AGNews~\citep{ZhZh15cha}. 
SST-2 classifies movie reviews by binary sentiment;
Hyperpartisan is a binary classification task on whether a news article takes an extreme partisan standpoint; and
AGNews classifies news articles into four different topics.
On SST-2 and AGNews, we perform the second-stage pretraining on the training inputs (not using the labels).
On Hyperpartisan, we use 100k unlabeled news articles that are released with the dataset.
For SST-2 and AGNews, we study a low-resource setting and set the number of finetuning examples to be $20$.
For Hyperpartisan, we use the training set, which has $515$ labeled examples.
All evaluations are performed by fine-tuning a \texttt{bert-base-uncased} model (See \Cref{app:details} for full details).

\textbf{Approximating \Specific Masking.}
We cannot identify the optimal set of \specific masks for an arbitrary downstream task, but these three tasks have associated lexicons which we can use to approximate the \specific masks.
For SST-2, we take the sentiment lexicon selected by \citet{HuLi04min}; for Hyperpartisan, we take the NRC word-emotion association lexicon~\citep{MoTu13cro}; and for AGNews, we extract topic words by training a logistic regression classifier and taking the top 1k features to be \specific masks. 

\textbf{Results.}
\Cref{fig:finetune_results} plots the finetuning performance of different masking strategies.
We observe that \textsc{Uniform} outperforms \textsc{Vanilla}, which indicates that second-stage pretraining is extracting useful information and our experiment setup is useful for studying how MLM leads to performance gains.
As expected, \textsc{Cloze} achieves the best accuracy, which confirms that \specific masks can be helpful and validates our cloze approximations.

The \textsc{Uniform} mask is much closer to \textsc{NoCloze} than \textsc{Cloze}. This suggests that uniform masking does not leverage \specific masks well and cloze reductions alone cannot account for the success of MLM.
This view is further supported by the observation that \textsc{NoCloze} outperforms \textsc{Vanilla} suggesting that \generic masks that are not \specific still contain useful inductive biases. 

Our results support our earlier view that there may be an alternative mechanism that allows \generic masks that are not \specific to benefit downstream learning.
Next, we will empirically examine BERT's learned conditional independence structure among tokens and show that the statistical dependencies relate to syntactic dependencies.

\subsection{Analysis: Unsupervised Parsing}
\label{sec:parsing}

\begin{figure*}[ht!]
    \centering
    \includegraphics[width=0.7\textwidth]{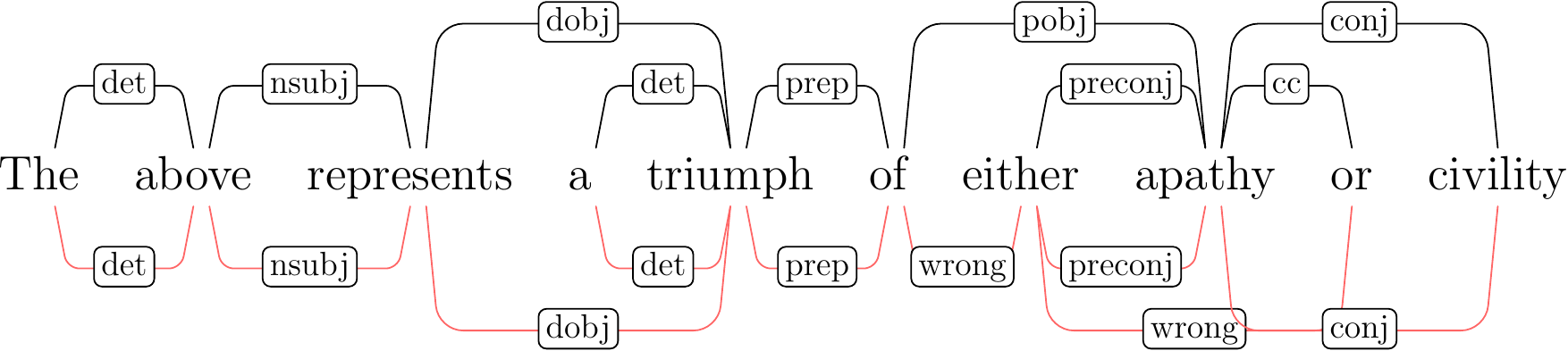}
    \caption{An example parse extracted from conditional MI. The black parse tree above the sentence represents the ground-truth parse and the red parse below is extracted from conditional MI. The correctly predicted edges are labeled with the annotated relations, and the incorrect ones are labeled as wrong.}
    \label{fig:parsing_example}
    \vspace{-15pt}
\end{figure*}

Our analysis in \cref{sec:gaussian_theory} shows that conditional MI (which is optimized by the MLM objective) can extract conditional independences.
We will show that statistical dependencies estimated by conditional MI are related to syntactic dependencies by using conditional MI for unsupervised parsing.

\textbf{Background.}
One might expect that the statistical dependencies among words are correlated with syntactic dependencies.
Indeed, \citet{FuQi19syn} show that heads and dependents in dependency parse trees have high pointwise mutual information (PMI) on average.
However, previous attempts~\citep{CaCh92two, Pa02gra} 
show that unsupervised parsing approaches based on PMI achieve close to random accuracy.
Our analysis suggests that MLMs extract a more fine-grained notion of statistical dependence (conditional MI) which does not suffer from the existence of latent variables (\Cref{prop:mi_entropy}). We now show that the conditional MI captured by MLMs achieves far better performance, on par with classic unsupervised parsing baselines.

\textbf{Baselines.}
We compare conditional MI to PMI as well as conditional PMI, an ablation in which we do not take expectation over possible words.
For all statistical dependency based methods (cond. MI, PMI, and cond. PMI), we compute pairwise dependence for each word pair in a sentence and construct a minimum spanning tree on the negative values to generate parse trees.
To contextualize our results, we compare against three simple baselines: \textsc{Random} which draws a random tree on the input sentence, \textsc{LinearChain} which links adjacent words in a sentence, and a classic unsupervised parsing method~\citep{KlMa04cor}.

\textbf{Experimental Setup.}
We conduct experiments on the English Penn Treebank using the WSJ corpus
and convert the annotated constituency parses to Stanford Dependency Formalism~\citep{MaMa06gen}.
Following \citet{YaJi20sec}, we evaluate on sentences of length $\leq 10$ in the test split, which contains 389 sentences (\Cref{app:more_parsing} describes the same experiment on longer sentences, which have similar results).
We experiment with the \texttt{bert-base-cased} model (more details in \Cref{app:details}) and evaluate by the undirected unlabeled attachment score (UUAS).

\begin{table}[t]
\centering
\begin{tabular}{cc}
    \toprule
    Method & UUAS \\
    \midrule
    \parsingmethod{Random} & $28.50 \pm 0.73$ \\
    \parsingmethod{LinearChain} & $54.13$ \\
    \makecell{\citet{KlMa04cor}}  & $55.91 \pm 0.68$ \\
    \midrule
    \parsingmethod{PMI} & $33.94$ \\
    \parsingmethod{Conditional PMI} & $52.44 \pm 0.19$ \\
    \parsingmethod{Conditional MI} & $\mathbf{58.74} \pm 0.22$ \\
    \bottomrule
\end{tabular}
\caption{Unlabeled Undirected Attachment Score on WSJ10 test split (section 23). Error bars show standard deviation across three random seeds.
}
\label{tab:parsing}
\vspace{-15pt}
\end{table}

\textbf{Results.} 
\Cref{tab:parsing} shows a much stronger-than-random association between conditional MI and dependency grammar.
In fact, the parses extracted from conditional MI has better quality than \textsc{LinearChain} and the classic method~\citep{KlMa04cor}.
Unlike conditional MI, PMI only has a close-to-random performance, which is consistent with prior work.
We also see that conditional MI outperforms conditional PMI, which is consistent with our theoretical framework that suggests that conditional MI (and not PMI) recovers the graphical model structure.

We also perform a fine-grained analysis by investigating relations where conditional MI differs from \textsc{LinearChain}.
Because the test split is small and conditional MI does not involve any training, we perform this analysis on $5{,}000$ sentences from the training split.
\Cref{tab:relation_analysis} presents the results and shows that conditional MI does not simply recover the linear chain bias.
Meanwhile, we also observe a deviation between conditional MI and dependency grammar on relations like \texttt{number} and \texttt{cc}.
This is reasonable because certain aspects of dependency grammar depend on human conventions that do not necessarily have a consensus~\citep{PoMa13coo}.

\begin{table}[t]
    \centering
    \begin{tabular}{ccc}
    \toprule
    Relation & Conditional MI & Linear Chain \\
    \midrule
    \texttt{xcomp} & \textbf{48.18}  & 9.93  \\
    \texttt{conj} & \textbf{43.36}  & 7.58  \\
    \texttt{dobj} & \textbf{58.96}  & 30.33  \\
    \midrule
    \texttt{number} & 50.55  & \textbf{92.62}  \\
    \texttt{quantmod} & 56.82  & \textbf{72.73}  \\
    \texttt{cc} & 31.39  & \textbf{41.10}  \\
    \bottomrule
\end{tabular}
    
    \caption{Six relations on which conditional MI disagrees with \textsc{LinearChain} under log odds ratio test with $p = 0{.}05$. A comprehensive list is in \Cref{app:details}.}
    \label{tab:relation_analysis}
    \vspace{-15pt}
\end{table}

\Cref{fig:parsing_example} illustrates with an example parse extracted from conditional MI.
We observe that conditional MI correctly captures \texttt{dobj} and \texttt{conj}.
Knowing the verb, \emph{e.g.}\ \textit{represents}, limits the range of objects that can appear in a sentence so intuitively we expect a high conditional MI between the direct object and the verb.
Similarly for phrases like \textit{``A and B''}, we would expect \textit{A} and \textit{B} to be statistically dependent.
However, conditional MI fails to capture \texttt{cc} (between \textit{apathy} and \textit{or}).
Instead, it links \textit{or} with \textit{either} which certainly has statistical dependence.
This once again suggests that the `errors' incurred by the conditional PMI method are not simply failures to estimate dependence but natural differences in the definition of dependence.

\section{Discussion and Conclusion}
We study how MLM with uniform masking can learn useful linguistic structures and inductive biases for downstream tasks.
Our work demonstrates that a substantial part of the performance gains of MLM pretraining cannot be attributed to task-specific, \specific masks.
Instead, learning with task-agnostic, \generic masks encourages the model to capture direct statistical dependencies among tokens, and we show through unsupervised parsing evaluations that this has a close correspondence to syntactic structures.
Existing work has suggested that statistical and syntactic dependencies are fundamentally different, with unsupervised parsing based on PMI achieving close-to-random performance.
Our work demonstrates that this is not necessarily the case, and better measures of statistical dependence (such as those learned by MLMs) can serve as implicit supervision for learning syntactic structures.
Our findings open new space for future works on how syntax can be learned in an emergent way and on how to design masking strategies that further improve dependency learning.

\bibliography{main}
\bibliographystyle{acl_natbib}

\clearpage
\onecolumn
\appendix
\section{Experimental Details}
\label{app:details}

\textbf{Experimental details for \Cref{sec:case_study}}
Our transformers have $2$ layers and for each transformer block, the hidden size and the intermediate size are both $64$.
We finetune the models for $10$ epochs and apply early stopping based on validation accuracy.
We use Adam~\citep{KiBa14ada} for optimization, using a learning rate of $1e^{-3}$ for pretraining and $1e^{-4}$ for finetuning.

\textbf{Experimental details for \Cref{sec:random_mask}}
\Cref{tab:data_stats} summarizes the dataset statistics of three real-world datasets we studied.
For second stage pretraining, we update the BERT model for $10$ epochs.
Following the suggestion in \citet{ZhWu20rev}, we finetune the pretrained BERT models for $400$ steps, using a batch size of $16$ and a learning rate of $1e^{-5}$.
We apply linear learning rate warmup for the first $10\%$ of finetuning and linear learning rate decay for the rest.
For SST-2 and AGNews, we average the results over $20$ random trials.
For Hyperpartisan, because the test set is small and the variation is larger, we average the results over $50$ random trials and evaluate on the union the development set and the test set for more stable results.

\begin{table}
    \centering
\setlength{\tabcolsep}{2pt}
\begin{tabular}{ccccc}
    \toprule
    Dataset & \# Classes & \# Pretrain & \# Finetune & \# Test \\
    \midrule
    SST-2 & 2 & 67k & 20 & 1.8k \\
    Hyperpartisan & 2 & 100k & 515 & 130 \\
    AGNews & 4 & 113k & 20 & 6.7k \\
    \bottomrule
\end{tabular}
    \caption{Specifications of datasets. For AGNews, we put away 6.7k as a development set.}
    \label{tab:data_stats}
\end{table}

\textbf{Experimental details for \Cref{sec:parsing}}
We convert the annotated constituency parses using the Stanford CoreNLP package~\citep{MaSu14sta}.
We compute conditional MI and conditional PMI using the \texttt{bert-base-cased} model and run Gibbs sampling for $2000$ steps.
BERT's tokenization may split a word into multiple word pieces.
We aggregate the dependencies between a word and multiple word pieces by taking the maximum value.
We compute the PMI statistics and train the K\&M model~\citep{KlMa04cor} on sentences of length $\leq 10$ in the WSJ train split (section 2-21).
For DMV, we train with the annotated POS tags using a public implementation released by \citep{HeNe18uns}.
Results are averaged over three runs when applicable.

\section{Additional Results}
\subsection{Additional Results in \Cref{sec:parsing}}
\label{app:more_parsing}
We conduct an additional experiment on the English Penn Treebank to verify that conditional MI can extract parses for sentences longer than ten words.
To expedite experimentation, we subsample $200$ out of $2416$ sentences from the test split of English Penn Treebank and the average sentence length of our subsampled dataset is $24.1$ words.
When applicable, we average over three random seeds and report standard deviations.
\Cref{tab:appendix_parsing} presents the UUAS of conditional MI and other methods.
We draw similar conclusions as in \Cref{sec:parsing}, observing that the parses drawn by conditional MI have higher quality than those of other baselines.

\Cref{tab:relation_analysis_more} presents a comprehensive list of relations on which Conditional MI disagrees with \textsc{LinearChina} under a log odds ratio test with $p=0.05$.

\begin{table}[t]
\centering
\begin{tabular}{cc}
    \toprule
    Method & UUAS \\
    \midrule
    \parsingmethod{Random} & $9.14 \pm 0.42$ \\
    \parsingmethod{LinearChain} & $47.69$ \\
    \makecell{\citet{KlMa04cor}}  & $48.76 \pm 0.24$ \\
    \midrule
    \parsingmethod{PMI} & $28.05$ \\
    \parsingmethod{Conditional PMI} & $44.75 \pm 0.09$ \\
    \parsingmethod{Conditional MI} & $\mathbf{50.62} \pm 0.38$ \\
    \bottomrule
\end{tabular}
\caption{Unlabeled Undirected Attachment Score on subsampled WSJ test split (section 23). Error bars show standard deviation across three random seeds.
}
\label{tab:appendix_parsing}
\end{table}

\begin{table}[t!]
    \centering
    \begin{tabular}{ccc}
    \toprule
    Relation & Conditional MI & Linear Chain \\
    \midrule
    \texttt{xcomp} & \textbf{48.18}  & 9.93  \\
    \texttt{conj} & \textbf{43.36}  & 7.58  \\
    \texttt{nsubjpass} & \textbf{33.81}  & 0.47  \\
    \texttt{dobj} & \textbf{58.96}  & 30.33  \\
    \texttt{mark} & \textbf{30.71}  & 9.45  \\
    \texttt{poss} & \textbf{58.63}  & 40.96  \\
    \texttt{ccomp} & \textbf{20.92}  & 4.18  \\
    \texttt{vmod} & \textbf{55.32}  & 41.84  \\
    \texttt{tmod} & \textbf{39.25}  & 27.68  \\
    \texttt{dep} & \textbf{50.15}  & 40.03  \\
    \texttt{pobj} & \textbf{48.68}  & 40.79  \\
    \texttt{nsub} & \textbf{55.87}  & 48.69  \\
    \midrule
    \texttt{number} & 50.55  & \textbf{92.62}  \\
    \texttt{possessive} & 72.00  & \textbf{97.78}  \\
    \texttt{pcomp} & 60.00  & \textbf{77.00}  \\
    \texttt{quantmod} & 56.82  & \textbf{72.73}  \\
    \texttt{appos} & 55.56  & \textbf{70.59}  \\
    \texttt{num} & 65.11  & \textbf{76.49}  \\
    \texttt{cc} & 31.39  & \textbf{41.10}  \\
    \texttt{prep} & 56.41  & \textbf{66.12}  \\
    \texttt{auxpass} & 75.00  & \textbf{83.26}  \\
    \texttt{nn} & 72.97  & \textbf{77.88}  \\
    \texttt{aux} & 55.49  & \textbf{59.66}  \\
    \bottomrule
\end{tabular}
    
    \caption{All relations on which Conditional MI disagree with \textsc{LinearChina} under a log odds ratio test with $p=0.05$.}
    \label{tab:relation_analysis_more}
\end{table}
\clearpage
\section{Proofs}
\label{app:proof}

\paragraph{Proof of Proposition \ref{prop:pca}}
We first recall our statement.

\pcaprop*

\begin{proof}
  
We will use the Davis-Kahan Theorem for our proof.

\newtheorem*{theorem*}{Theorem}
\begin{theorem*}[Davis-Kahan~\citep{StSu90mat}]
  Let $\sigma$ be the eigengap between the $k$th and the $k{+}1$th eigenvalue of two positive semidefinite symmetric matrices $\Sigma$ and $\Sigma'$. Also, let $V$ and $V'$ be the first $k$ eigenvectors of $\Sigma$ and $\Sigma'$ respectively.
  Then we have,
  $$
  \frac{1}{\sqrt{2}} \norm{VV^{\top} - V'{V'}^{\top}}_{op} \leq \frac{\norm{\Sigma - \Sigma}_{op}}{\sigma}.
  $$
  That is, we can bound the error in the subspace projection in terms of the matrix perturbation.
\end{theorem*}

In our setting, we choose $\Sigma = A \sigmazz A^{\top} + \sigmaxx$ and $\Sigma' = A \sigmazz A^{\top}$.
We know the eigengap of $\Sigma'$ is $\lambda_{k}$ because $\Sigma'$ only has $k$ nonzero eigenvalues.
By Weyl's inequality, the $k$th eigenvalue is at most perturbed by $\lambda_{\rmX\rmX, k+1}$, which is the $k{+}1$ eigenvalue of $\sigmaxx$.
Let $V$ be the top $k$ eigenvectors of $\Sigma'$ and assuming $\lambda_k > \lambda_{\rmX\rmX, k{+}1}$, we have,
\begin{align*}
  \frac{1}{\sqrt{2}} \norm{AA^{\top} - V{V}^{\top}}_{op} \leq & \frac{\norm{\Sigma - \Sigma'}_{op}}{\lambda_{k} - \lambda_{\rmX\rmX, k+1}} \\
  = & \frac{\norm{\sigmaxx}_{op}}{\lambda_{k} - \lambda_{\rmX\rmX, k+1}}.
\end{align*}

Turning this operator norm bound into approximation bound, we have
\begin{align*}
  \E_{\rmX} \norm{A\rmZ - \xpca}_{2} = 
  & \E_{\rmX} \norm{AA^{\top}A\rmZ - V{V}^{\top}\rmX}_{2} \\
  = & \E_{\rmX} \norm{AA^{\top}A\rmZ - VV^{\top}A\rmZ + VV^{\top}A\rmZ - V{V}^{\top}\rmX}_{2} \\
  \leq & \E_{\rmX} \norm{AA^{\top}A\rmZ - VV^{\top}A\rmZ}_{2} + \norm{VV^{\top}A\rmZ - V{V}^{\top}\rmX}_{2} \\
  \leq & \E_{\rmX} \norm{AA^{\top}A\rmZ - VV^{\top}A\rmZ}_{2} + \norm{VV^{\top} (A\rmZ - \rmX)}_{2} \\
  \leq & \E_{\rmX} \norm{AA^{\top} - VV^{\top}}_{op} \cdot \norm{A\rmZ}_{2} + \norm{VV^{\top}}_{op} \norm{A\rmZ - \rmX}_{2}. \\
  = & \E_{\rmX} \norm{AA^{\top} - VV^{\top}}_{op} \cdot \norm{A\rmZ}_{2} + \norm{AA^{\top} + VV^{\top} - AA}_{op} \norm{A\rmZ - \rmX}_{2}\\
  \leq & \E_{\rmX} \norm{AA^{\top} - VV^{\top}}_{op} \cdot \norm{A\rmZ}_{2} + (\norm{AA^{\top}}_{op}+\norm{VV^{\top} - AA}_{op}) \norm{A\rmZ - \rmX}_{2}\\
  = & \E_{\rmX} \norm{AA^{\top} - VV^{\top}}_{op} \cdot (\norm{A\rmZ}_{2} + \norm{A\rmZ - \rmX}_{2}) + \norm{AA^{\top}}_{op} \norm{A\rmZ - \rmX}_{2}.\\
\end{align*}

We use the fact that $\E_{\rmX, \rmZ} \norm{A\rmZ - \rmX}^2_{2} = \mathrm{tr}(\sigmaxx)$ and Jensen's inequality to bound,
$$
\E_{\rmX} \norm{A\rmZ - \rmX}_{2} \leq \sqrt{\mathrm{tr}(\sigmaxx)}.
$$
Combining these inequalities, we have
\begin{align*}
& \E_{\rmX} \norm{A\rmZ - \xpca}_{2} \\
\leq\; & \frac{\sqrt{2}\norm{\sigmaxx}_{op}}{\lambda_{k} - \lambda_{\rmX\rmX, k+1}} \cdot (\norm{A\rmZ}_{2} + 
\sqrt{\mathrm{tr}(\sigmaxx)}) + \norm{AA^{\top}}_{op}\sqrt{\mathrm{tr}(\sigmaxx)}
\end{align*}

\end{proof}

\paragraph{Proof of Proposition \ref{prop:regression}}
We first recall our statement.

\regressionprop*

\begin{proof}
Let $\anoi \in \mathbb{R}^{d{-}1 \times k}$ be the matrix where we omit the $i$th row of $A$ and $\atakei \in \mathbb{R}^{k}$ be the $i$th row of $A$.
Let $\sigmaxxnoi \in \mathbb{R}^{d{-}1 \times d{-}1}$ be the matrix where we omit the $i$th row and $i$th column of $\sigmaxx$, and $\sigmaxxtakei \in \mathbb{R}^{d{-}1}$ be the vector formed by dropping the $i$th row and taking the $i$th column of $\sigmaxx$.
Similarly, denote $\rmX_{\setminus i} \in \mathbb{R}^{d{-}1}$ be the vector where we omit the $i$ coordinate of $X$.

We start by writing down the expression of $\betatwoslsi$.
Recall that the Least Squares regression between two zero-mean Gaussian variables $\rmX$ and $\rmY$ can be written as 
$$\beta = Cov(\rmX, \rmY)Cov(\rmX, \rmX)^{-1},$$
where $Cov(\rmX, \rmX)$ is the covariance matrix of $X$ and we assume it is full rank.
Since $Cov(\xnoi, \rmZ)$ is $\anoi\sigmazz$, we can write the coefficient of regression from $\xnoi$ to $\rmZ$ as
$$\beta_{\xnoi \rightarrow \rmZ} = \sigmazz \anoi^{\top}(\anoi \sigmazz \anoi^{\top} + \sigmaxxnoi)^{-1}$$
and by assumption we have $\beta_{\rmZ \rightarrow \rvx_i} = \atakei$.
So we can write down
$$\betatwoslsi = \atakei \sigmazz \anoi^{\top}(\anoi \sigmazz \anoi^{\top} + \sigmaxxnoi)^{-1}.$$
Now we consider masked regression for the $i$th coordinate, $\rvx_i$,
$$
\betamaskone = (\atakei \sigmazz \anoi^{\top} + \sigmaxxtakei)(\anoi \sigmazz \anoi^{\top} + \sigmaxxnoi)^{-1}.
$$


Comparing $\betatwosls$ and $\betamaskone$, we observe that the second term is the same and the key is to bound the first term.
Consider the error term between the coefficients,
\begin{align*}
& \norm{\sigmaxxtakei (\anoi \sigmazz \anoi^{\top} + \sigmaxxnoi)^{-1}}_{2} \\
\leq\;& \norm{\sigmaxxtakei}_{2} \norm{(\anoi \sigmazz \anoi^{\top} + \sigmaxxnoi)^{-1}}_{op} \\
\leq\;& \norm{\sigmaxxtakei}_{2} \norm{(A \sigmazz A^{\top} + \sigmaxx)^{-1}}_{op}.
\end{align*}
That is, the error term scales with the off-diagonal terms $\norm{\sigmaxxtakei}_{2}$.

Converting our bound on the error term into an approximation bound, we have
$$
\xmaskone = \betatwoslsi \rmX + O(\norm{\sigmaxxtakei}_{2}).
$$

\end{proof}

\paragraph{Proof for Proposition \ref{prop:mi_entropy}. }
\entropyprop*

\begin{proof}
  The proof follows from the definition of conditional mutual information.
  Denote $H(\cdot)$ as the entropy function.

  We start by observing that
  \begin{align*}
  I(x_i; x_j | \rmZ, \masksequencetwo) =\; & I(x_i;x_j|\masksequencetwo) - I(x_i;\rmZ | \masksequencetwo) + I(x_i;\rmZ|x_j, \masksequencetwo) \\
  & \text{(Through chain rule of mutual information.)} \\
  =\;& I(x_i;x_j|\masksequencetwo) + H(\rmZ|x_i, \masksequencetwo) - H(\rmZ | \masksequencetwo) \\
  &+ H(\rmZ|x_j, \masksequencetwo) - H(\rmZ|x_i,x_j, \masksequencetwo). 
  \end{align*}
  Then we have,
  \begin{align*}
    & I(x_i; x_j | \masksequencetwo) - I(x_i; x_j | \rmZ, \masksequencetwo)\\
    =\; &  - H(\rmZ|x_i, \masksequencetwo) + H(\rmZ | \masksequencetwo) -  H(\rmZ|x_j, \masksequencetwo) + H(\rmZ|x_i,x_j, \masksequencetwo) \\
    \leq \; & H(\rmZ | \masksequencetwo) + H(\rmZ|x_i,x_j, \masksequencetwo) \\
    \leq \; & 2 \cdot H(\rmZ | \masksequencetwo).
  \end{align*}
\end{proof}

\estimatorprop*

\begin{proof}
  Expanding the definition of mutual information, we write
  \begin{align*}
  I(x_i; x_j | \masksequencetwo)-\hat{I}_{\theta}(x_i; x_j | \masksequencetwo) 
  =\;& \E_{x_j} [\dkl{p(x_i | x_j, \masksequencetwo)}{p_{\theta}(x_i | x_j, \masksequencetwo)}] -\\
  & \dkl{\E_{x_j}p(x_i | x_j, \masksequencetwo)}{\E_{x_j}p_{\theta}(x_i | x_j, \masksequencetwo)}.
  \end{align*}
  Dropping the the second term, we have
  \begin{align*}
  \hat{I}_{\theta}(x_i; x_j | \masksequencetwo)-I(x_i; x_j | \masksequencetwo)
  \geq - \E_{x_j} [\dkl{p(x_i | x_j, \masksequencetwo)}{p_{\theta}(x_i | x_j, \masksequencetwo)}].
  \end{align*}
  Dropping the the first term, we have
  \begin{align*}
  &I(x_i; x_j | \masksequencetwo)-\hat{I}_{\theta}(x_i; x_j | \masksequencetwo) \\
  \leq\;& 
  \dkl{\E_{x_j}p(x_i | x_j, \masksequencetwo)}{\E_{x_j}p_{\theta}(x_i | x_j, \masksequencetwo)} \\
  \leq\;&
  \E_{x_j} \dkl{p(x_i | x_j, \masksequencetwo)}{p_{\theta}(x_i | x_j, \masksequencetwo)},
  \end{align*}
  which uses the convexity of KL-divergence and Jensen's inequality.
\end{proof}

\end{document}